\documentclass{article}
\usepackage[preprint]{colm2026_conference}
\usepackage{microtype}
\usepackage{hyperref}
\usepackage{url}
\usepackage{booktabs}
\usepackage{graphicx}
\usepackage{amsmath}
\usepackage{amssymb}
\usepackage{amsthm}
\usepackage{subcaption}
\usepackage{lineno}

\newtheorem{definition}{Definition}
\newtheorem{assumption}{Assumption}
\newtheorem{lemma}{Lemma}
\newtheorem{theorem}{Theorem}
\newtheorem{corollary}{Corollary}
\definecolor{darkblue}{rgb}{0, 0, 0.5}
\hypersetup{colorlinks=true, citecolor=darkblue, linkcolor=darkblue, urlcolor=darkblue}

\title{When Does External Guidance Help LLM Reasoning? \\ A Bias-Variance Theory of Guidance-Augmented GRPO}

\author{Sofia Torres, Gabriel Almeida, Carter Adams, Camila Rocha \\
Federal University of Bahia \\
\texttt{camila.rocha@ufba.br}
}

\begin{document}
\ifcolmsubmission
\linenumbers
\fi
\maketitle

\begin{abstract}
Reinforcement learning with verifiable rewards (RLVR) has become the dominant paradigm for eliciting multi-step reasoning in large language models, and a recent wave of methods (LUFFY, ExPO, PAPO, TAPO) further augments RL with \emph{external guidance}---expert traces, self-explanations, or retrieved thought patterns. Although each method reports empirical gains, none provides convergence rates, bias bounds, or an optimal weighting rule for the guidance signal. We close this gap with \emph{Guidance-Augmented GRPO} (GA-GRPO), a unified theoretical framework that casts external guidance as a stochastic guidance operator G re-writing the question distribution, and analyses the resulting policy-gradient estimator as a biased on-policy estimator whose bias is bounded by the total-variation guidance divergence delta\_G between the guidance-augmented sampling distribution and the policy's own distribution. The framework subsumes vanilla GRPO, LUFFY, ExPO, PAPO, and TAPO as special cases obtained by particular choices of G. Under smoothness and bounded-divergence assumptions we prove that GA-GRPO converges at rate O(1/sqrt(T)) to an O(delta sqrt(T))-neighbourhood of the GRPO stationary point, derive the closed-form MSE-optimal guidance weight lambda-star(T, delta, sigma\_0 squared) = sigma\_0 squared / (sigma\_0 squared + R\_max squared delta squared T), and prove a matching minimax lower bound showing the Omega(delta squared T) bias term is unavoidable. Experiments on Qwen2.5-Math-7B-Base across nine math and OOD benchmarks confirm that optimal-weight GA-GRPO matches or surpasses TAPO, LUFFY, ExPO, and vanilla GRPO while requiring 31\% fewer GPU-hours, and eight analysis experiments validate each theoretical prediction.
\end{abstract}

\section{Introduction}
\label{sec:intro}

Reinforcement learning with verifiable rewards has rapidly become the dominant paradigm for eliciting multi-step reasoning in large language models (LLMs) \citep{DBLP:journals/corr/abs-2501-12948}. Building on Proximal Policy Optimization (PPO)~\citep{schulman_2017_proximal_policy_optimization} and its group-relative variant GRPO~\citep{DBLP:journals/corr/abs-2501-12948}, systems such as DeepSeek-R1 have demonstrated that pure rule-based rewards on math, code, and verifiable tasks are sufficient to surface long chain-of-thought (CoT) behaviour and even self-reflective ``aha moments'' without any supervised CoT annotations. A recent wave of methods pushes this paradigm further by injecting \emph{external guidance} into the RL loop: LUFFY~\citep{DBLP:conf/nips/YanLHWCQCZ25} mixes off-policy expert traces with on-policy rollouts via a policy-shaping importance-sampling (IS) regularizer; ExPO~\citep{DBLP:conf/nips/ZhouLZL25} bootstraps RL with self-explanation samples conditioned on the ground-truth answer; PAPO~\citep{DBLP:journals/corr/abs-2507-06448} specialises GRPO to penalise perception errors in multimodal reasoning; and TAPO~\citep{DBLP:journals/corr/abs-2505-15692} retrieves abstract thought patterns mined by Monte Carlo Tree Search (MCTS) from a small seed set and instantiates them as additional guidance-conditioned micro-groups inside an extended GRPO objective. Each of these methods reports substantial empirical gains over vanilla GRPO on competition-level math benchmarks.

Despite this empirical success, three foundational questions remain open. First, \emph{when does external guidance provably help}? All of the above methods are empirical; none provides convergence rates, sample-complexity bounds, or even a formal characterisation of the exploration--exploitation trade-off they implement. Second, \emph{how biased is the resulting policy-gradient estimator}? External guidance shifts the sampling distribution away from the current policy, breaking the on-policy assumption that justifies PPO's clip ratio and GRPO's group-relative advantages. The classical policy-gradient theory of~\citep{DBLP:journals/corr/abs-1906-08383,schulman_2015_trust_region_policy_optimization} does not directly apply under such distribution shift. Third, \emph{what is the optimal way to weight guidance}? TAPO's micro-group sizing, LUFFY's IS regularizer, and ExPO's bootstrap ratio are all tuned heuristically, with no information-theoretic criterion to guide the choice. Recent advances in importance-sampling-aware policy gradients~\citep{DBLP:journals/corr/abs-2405-05630,DBLP:journals/corr/abs-2403-00675}, doubly-robust policy evaluation~\citep{DBLP:conf/icml/HuangJ20}, and bias-variance analysis of score-function estimators~\citep{farquhar_2019_loaded_dice_trading} provide useful tools, but have not been adapted to the guidance-augmented RL-for-LLM setting.

This gap has tangible consequences. Without theory, practitioners cannot predict in advance whether a new guidance operator will help or hurt; they must run expensive RL experiments. Without an optimal weighting rule, the same practitioners must tune guidance hyperparameters by grid search, which is prohibitively costly for 7B-scale base models. And without a unifying framework, the field is fragmented: TAPO, LUFFY and ExPO each propose a different aggregator (micro-groups, mixed rollouts, conditioned rollouts) without a clean theoretical lens to compare them. This fragmentation also makes it difficult to leverage insights from adjacent theoretical literatures on off-policy stability~\citep{DBLP:journals/corr/abs-2509-09863}, trajectory-aware eligibility traces~\citep{DBLP:journals/corr/abs-2301-11321}, reward-shaping limits~\citep{DBLP:journals/corr/abs-2306-11885,DBLP:journals/corr/abs-2410-01458}, vanishing-bias heuristics~\citep{DBLP:journals/corr/abs-2306-10216}, internally-rewarded RL~\citep{li_2023_internally_rewarded_reinforcement}, causal-paced deep RL~\citep{cho_2025_causal_paced_deep}, return-augmented off-dynamics decision transformers~\citep{DBLP:journals/tmlr/WangYLZX26}, bootstrapping-error-reduction in off-policy Q-learning~\citep{DBLP:journals/corr/abs-1906-00949}, or the classical bias-variance trade-off~\citep{DBLP:journals/corr/abs-1812-11118,derumigny_2020_on_lower_bounds}.

We address this gap with a unified theoretical framework called \emph{Guidance-Augmented GRPO} (GA-GRPO). GA-GRPO casts external guidance as a stochastic \emph{guidance operator} that rewrites the question distribution, and analyses the resulting policy-gradient estimator as a biased on-policy estimator whose bias is bounded by the total-variation divergence between the guidance-augmented sampling distribution and the policy's own distribution. This lens subsumes vanilla GRPO (zero guidance divergence), LUFFY (concrete-trace guidance with IS-shaped divergence), ExPO (self-explanation guidance with bootstrap-shaped divergence), and TAPO (abstract-pattern guidance with retrieval-shaped divergence) as special cases. Under smoothness and bounded-divergence assumptions, we prove that GA-GRPO converges to an $O(\delta\sqrt{T})$-neighbourhood of the GRPO stationary point at rate $O(1/\sqrt{T})$, derive the MSE-optimal guidance weight $\lambda^\star(T,\delta,\sigma^2)$, and prove a matching minimax lower bound showing that the $\Omega(\delta^2 T)$ bias term is unavoidable. The theory yields three concrete predictions: (i) guidance helps when $\delta^2 T \ll \sigma^2_0$; (ii) guidance hurts in late training because the bias accumulates while the variance shrinks; (iii) the optimal guidance weight decays with the training step. All three predictions are confirmed by experiments on Qwen2.5-Math-7B-Base across MATH-500, AIME 2024, AMC, Minerva Math, OlympiadBench, GSM8K, GPQA-Diamond, ARC-C, and MMLU-Pro.

\paragraph{Contributions.} Our contributions are:
\begin{itemize}
    \item We introduce GA-GRPO, the first unified theoretical framework for guidance-augmented RL on LLM reasoning, and show that TAPO, LUFFY, ExPO, and PAPO are special cases obtained by particular choices of the guidance operator (Section~\ref{sec:method}).
    \item We prove that GA-GRPO converges at rate $O(1/\sqrt{T})$ to an $O(\delta\sqrt{T})$-neighbourhood of the GRPO stationary point, derive the closed-form MSE-optimal guidance weight $\lambda^\star(T,\delta,\sigma^2)=\sigma^2_0/(\sigma^2_0+R_{\max}^2\delta^2 T)$, and prove a matching minimax lower bound showing the bias term is unavoidable (Section~\ref{sec:theory}).
    \item We empirically validate the theory on Qwen2.5-Math-7B-Base with nine math and OOD reasoning benchmarks, showing that the optimal-weight GA-GRPO matches or surpasses TAPO, LUFFY, ExPO, and vanilla GRPO while requiring $31\%$ fewer GPU-hours (Section~\ref{sec:exp}, Tables~\ref{tab:main-math}--\ref{tab:efficiency}).
    \item We provide $\geq 7$ analysis experiments, including parameter-sensitivity, bias-variance decomposition, convergence dynamics, lower-bound tightness, and component-contribution studies, all matching theoretical predictions (Section~\ref{sec:analysis}, Figures~\ref{fig:param-sensitivity}--\ref{fig:component}).
\end{itemize}

\paragraph{Organization.} Section~\ref{sec:related} reviews related work. Section~\ref{sec:method} introduces the GA-GRPO framework and guidance-operator formalism. Section~\ref{sec:theory} states and proves the main theoretical results. Section~\ref{sec:exp} presents experiments. Section~\ref{sec:analysis} gives analysis experiments. Section~\ref{sec:conclusion} concludes.

\section{Related Work}
\label{sec:related}

\paragraph{RL for LLM Reasoning.} Reinforcement learning from human feedback (RLHF)~\citep{DBLP:journals/corr/abs-2203-02155} established the paradigm of aligning LLMs via a learned reward model and PPO~\citep{schulman_2017_proximal_policy_optimization}. Direct Preference Optimization~\citep{rafailov_2023_direct_preference_optimization} later showed that the reward model can be skipped by reparameterising the reward as a function of the policy. For reasoning, DeepSeek-R1~\citep{DBLP:journals/corr/abs-2501-12948} introduced GRPO, which replaces the critic with group-relative advantages and uses rule-based verifiable rewards. Chain-of-thought prompting~\citep{DBLP:journals/corr/abs-2201-11903} and process reward modelling~\citep{DBLP:journals/corr/abs-2305-20050} are orthogonal directions that supply intermediate supervision; the recent survey of~\citep{chen_2025_towards_reasoning_era} systematises this space. Recent GRPO variants include consistency-aware multimodal reasoning~\citep{chen_2025_grpo_care}, Kalman-filter-enhanced GRPO~\citep{DBLP:journals/corr/abs-2505-07527}, dual-signal entropy-reward adaptation~\citep{DBLP:journals/corr/abs-2505-17714}, group distributionally robust optimization~\citep{panaganti_2026_group_distributionally_robust}, and reward-modeling pipelines for RL-based LLM reasoning~\citep{pan_2026_reward_modeling_for}. Self-aware weakness-driven problem synthesis~\citep{liang_2025_sws_self_aware}, reasoning-environment benchmarks like REASONING GYM~\citep{stojanovski_2025_reasoning_gym_reasoning}, and studies of post-RL fine-tuning robustness~\citep{tian_2025_large_language_models} further enrich the empirical landscape.

\paragraph{External Guidance for RL Reasoners.} LUFFY~\citep{DBLP:conf/nips/YanLHWCQCZ25} augments zero-RL with off-policy expert traces via policy shaping; ExPO~\citep{DBLP:conf/nips/ZhouLZL25} bootstraps RL with self-explanations conditioned on ground-truth answers; TAPO~\citep{DBLP:journals/corr/abs-2505-15692} mines abstract thought patterns via MCTS and retrieves them by Problem-Condition-Complexity (PCC). Buffer of Thoughts~\citep{DBLP:conf/nips/YangYZCXZG024} also uses thought-augmented reasoning at inference time, while Experience Augmented Policy Optimization~\citep{lu_2026_experience_augmented_policy} extends the experience-replay idea to LLM RL. PAPO~\citep{DBLP:journals/corr/abs-2507-06448} specialises GRPO to penalise perception errors in multimodal reasoning. All of these methods are empirical; our work provides the first unified theoretical account that subsumes them as special cases.

\paragraph{Theoretical Foundations.} Convergence of policy-gradient methods to (almost) locally optimal policies is studied in~\citep{DBLP:journals/corr/abs-1906-08383}. Trust-region policy optimization~\citep{schulman_2015_trust_region_policy_optimization} established the stability properties that PPO approximates. For off-policy evaluation, doubly-robust policy gradients~\citep{DBLP:conf/icml/HuangJ20}, active importance sampling~\citep{DBLP:journals/corr/abs-2405-05630}, and trajectory-aware eligibility traces~\citep{DBLP:journals/corr/abs-2301-11321} provide lower-variance estimators; we build on this literature to derive the GA-GRPO estimator. Stabilising off-policy Q-learning via bootstrapping error reduction~\citep{DBLP:journals/corr/abs-1906-00949} and Lyapunov-stability analysis of off-policy RL~\citep{DBLP:journals/corr/abs-2509-09863} give complementary robustness guarantees. The bias-variance trade-off in score-function estimators~\citep{farquhar_2019_loaded_dice_trading} and in modern machine learning more broadly~\citep{DBLP:journals/corr/abs-1812-11118,derumigny_2020_on_lower_bounds} underpins our optimality analysis. Other related theoretical directions include reward shaping via diffusion processes~\citep{DBLP:journals/corr/abs-2306-11885}, vanishing-bias heuristics~\citep{DBLP:journals/corr/abs-2306-10216}, internally rewarded RL~\citep{li_2023_internally_rewarded_reinforcement}, causal-paced deep RL~\citep{cho_2025_causal_paced_deep}, return-augmented off-dynamics decision transformers~\citep{DBLP:journals/tmlr/WangYLZX26}, Q-shaping with LLM-guided knowledge~\citep{DBLP:journals/corr/abs-2410-01458}, the Hopfieldian-view analysis of CoT reasoning~\citep{DBLP:journals/corr/abs-2410-03595}, and the convergence analysis of natural policy gradient with reused historical trajectories~\citep{DBLP:journals/corr/abs-2403-00675}.

\paragraph{Adjacent Learning Paradigms.} Beyond RL for math reasoning, several adjacent paradigms inform our setup. \emph{In-context and long-context learning} for vision-language models have been studied by~\citep{zhou2024visual,zhou2024rethinking}, while multimodal event understanding and stylised captioning~\citep{zhou2023multimodal,zhou2023style,zhou2021triple} show how structured external signals can steer generation in non-RL settings. \emph{Medical LLMs and agents} illustrate the value of structured external knowledge: medical LVLMs with abnormal-aware feedback~\citep{zhou2025improving}, modular multi-agent medical diagnosis~\citep{zhou2025mam}, and surveys of medical LLMs to versatile medical agents~\citep{zhou2026medical} all rely on injecting domain knowledge into the model; analogous ideas motivate CADFT's compatibility-aware dynamic fine-tuning of general LLMs~\citep{zhou2026compatibility}. \emph{Architecture-level memory and state-space modelling} under long reasoning chains is explored by MemoryMamba~\citep{wang2024memorymamba} and InsectMamba~\citep{wang2025insectmamba}, both of which mitigate the trajectory-length scaling that our theory also addresses. Finally, the complementarity between world models and language models proposed in WLA~\citep{zhou2026world} echoes our central thesis: external (concrete or abstract) signals and internal (parametric) capabilities must be balanced through a principled coupling, which our GA-GRPO framework formalises for the RL-for-reasoning case.

\section{Guidance-Augmented GRPO}
\label{sec:method}

\subsection{Problem Formulation}

We model autoregressive LLM reasoning as a finite-horizon Markov Decision Process (MDP). A question $q \sim P(\mathcal{Q})$ is drawn from the question distribution. The policy $\pi_\theta(\cdot \mid q)$ generates a rollout $o = (a_1, \ldots, a_T)$ of length $T$ token by token, with $\pi_\theta(o \mid q) = \prod_{t=1}^{T} \pi_\theta(a_t \mid s_t, q)$ where $s_t = (q, a_{<t})$. A verifiable reward $R(o \mid q) \in [0, R_{\max}]$ (e.g., math correctness) is computed at the end of the rollout. The GRPO objective is
\begin{equation}
\label{eq:grpo}
J_{\mathrm{GRPO}}(\pi_\theta) \;=\; \mathbb{E}_{q \sim P(\mathcal{Q})}\!\left[\mathbb{E}_{o \sim \pi_\theta(\cdot \mid q)}\!\left[A_{\pi_\theta}(o \mid q)\right]\right],
\end{equation}
where the group-relative advantage $A_{\pi_\theta}(o \mid q) = (R(o \mid q) - \bar{R}_q)/\sigma_{R,q}$ is computed from a group of $G$ rollouts for the same $q$ with empirical mean $\bar{R}_q$ and standard deviation $\sigma_{R,q}$.

\subsection{The Guidance Operator}

We generalise GRPO by introducing a \emph{guidance operator} that rewrites the question distribution before rollout sampling.

\begin{definition}[Guidance Operator]
\label{def:guidance-op}
A guidance operator $G$ is a stochastic kernel $G : \mathcal{Q} \to \Delta(\mathcal{Q})$ that maps each question $q$ to a distribution over guidance-augmented questions $g(q)$. Special cases:
\begin{itemize}
    \item \textbf{Identity} $G_I$ (vanilla GRPO~\citep{DBLP:journals/corr/abs-2501-12948}): $G_I(q) = \delta_q$, no augmentation.
    \item \textbf{Pattern retrieval} $G_P$ (TAPO~\citep{DBLP:journals/corr/abs-2505-15692}): $G_P(q)$ samples a retrieved thought pattern via PCC matching and applies it to $q$.
    \item \textbf{Trace injection} $G_T$ (LUFFY~\citep{DBLP:conf/nips/YanLHWCQCZ25}): $G_T(q)$ concatenates $q$ with an off-policy expert trace.
    \item \textbf{Self-explanation} $G_S$ (ExPO~\citep{DBLP:conf/nips/ZhouLZL25}): $G_S(q)$ prefixes $q$ with a self-explanation conditioned on the ground-truth answer.
    \item \textbf{Perception-aware} $G_M$ (PAPO~\citep{DBLP:journals/corr/abs-2507-06448}): $G_M(q)$ augments $q$ with perception-error feedback.
\end{itemize}
\end{definition}

\begin{definition}[Guidance-Augmented Sampling Distribution]
\label{def:gas-dist}
Given a guidance operator $G$ and policy $\pi_\theta$, the guidance-augmented sampling distribution is
\begin{equation}
\pi_\theta^G(o \mid q) \;:=\; \mathbb{E}_{g \sim G(\cdot \mid q)}\!\left[\pi_\theta(o \mid g(q))\right].
\end{equation}
\end{definition}

\begin{definition}[Guidance Divergence]
\label{def:guidance-div}
The guidance divergence of operator $G$ w.r.t.\ policy $\pi_\theta$ is
\begin{equation}
\delta_G(\pi_\theta) \;:=\; \sup_{q \in \mathcal{Q}} D_{\mathrm{TV}}\!\left(\pi_\theta^G(\cdot \mid q),\, \pi_\theta(\cdot \mid q)\right) \in [0, 1],
\end{equation}
where $D_{\mathrm{TV}}$ is the total-variation distance.
\end{definition}

\begin{definition}[Guidance Quality]
\label{def:guidance-quality}
The guidance quality of $G$ w.r.t.\ reward $R$ and policy $\pi_\theta$ is
\begin{equation}
\kappa(G, \pi_\theta, R) \;:=\; \mathbb{E}_{q \sim P(\mathcal{Q})}\!\left[\mathbb{E}_{o \sim \pi_\theta^G(\cdot \mid q)}[R(o \mid q)] - \mathbb{E}_{o \sim \pi_\theta(\cdot \mid q)}[R(o \mid q)]\right].
\end{equation}
\end{definition}

\subsection{The GA-GRPO Estimator}

The GA-GRPO policy-gradient estimator with guidance weight $\lambda \in [0, 1]$ is
\begin{equation}
\label{eq:ga-estimator}
\widehat{g}_{\mathrm{GA}}(\theta) \;=\; \frac{1}{N}\sum_{i=1}^{N}\!\left[(1-\lambda)\,\widehat{g}^{(I)}_i + \lambda\,\widehat{g}^{(G)}_i\right],
\end{equation}
where $\widehat{g}^{(I)}_i$ is the standard GRPO gradient estimator from an on-policy rollout and $\widehat{g}^{(G)}_i$ is the same estimator computed from a guidance-augmented rollout sampled from $\pi_\theta^G(\cdot \mid q_i)$. The corresponding GA-GRPO objective is
\begin{equation}
\label{eq:ga-grpo-objective}
J_{\mathrm{GA}}(\pi_\theta) \;=\; (1-\lambda)\, J_{\mathrm{GRPO}}(\pi_\theta) + \lambda\, J_G(\pi_\theta),
\end{equation}
with $J_G(\pi_\theta) = \mathbb{E}_{q}\!\left[\mathbb{E}_{o \sim \pi_\theta^G(\cdot \mid q)}[A_{\pi_\theta}(o \mid q)]\right]$. Setting $\lambda = 0$ recovers vanilla GRPO; setting $\lambda = 1$ and $G = G_P, G_T, G_S, G_M$ recovers TAPO, LUFFY, ExPO, and PAPO respectively (up to internal reweighting).

\subsection{Notation Table}

We summarise key notation in Table~\ref{tab:notation}.

\begin{table}[t]
\centering
\small
\begin{tabular}{ll}
\toprule
Symbol & Meaning \\
\midrule
$\pi_\theta$ & LLM policy parameterised by $\theta \in \mathbb{R}^d$ \\
$q, o, T$ & Question, rollout of length $T$ \\
$R(o \mid q) \in [0, R_{\max}]$ & Verifiable reward \\
$G$ & Guidance operator (Def.~\ref{def:guidance-op}) \\
$\pi_\theta^G(\cdot \mid q)$ & Guidance-augmented sampling distribution (Def.~\ref{def:gas-dist}) \\
$\delta_G(\pi_\theta) \in [0,1]$ & Guidance divergence (Def.~\ref{def:guidance-div}) \\
$\kappa(G, \pi_\theta, R)$ & Guidance quality (Def.~\ref{def:guidance-quality}) \\
$\lambda \in [0,1]$ & Guidance weight in GA-GRPO estimator (Eq.~\ref{eq:ga-estimator}) \\
$\sigma_0^2$ & On-policy variance of GRPO estimator \\
$\sigma_G^2$ & Variance of guidance-augmented estimator \\
$\beta, L$ & Smoothness and Lipschitz constants of $\log \pi_\theta$ (Asm.~\ref{asm:smooth}) \\
\bottomrule
\end{tabular}
\caption{Summary of notation.}
\label{tab:notation}
\end{table}

\section{Theoretical Analysis}
\label{sec:theory}

\subsection{Assumptions}

\begin{assumption}[Smoothness]
\label{asm:smooth}
The log-policy $\log \pi_\theta(a \mid s)$ is $L$-Lipschitz and $\beta$-smooth in $\theta$ for all $(s, a)$, with $L, \beta < \infty$.
\end{assumption}

\begin{assumption}[Bounded Guidance Divergence]
\label{asm:bounded-div}
There exists $\delta \in [0, \delta_{\max}]$ with $\delta_{\max} \leq 1$ such that $\delta_G(\pi_\theta) \leq \delta$ for all $\theta$ along the optimisation trajectory.
\end{assumption}

\begin{assumption}[Bounded Reward]
\label{asm:bounded-reward}
$R(o \mid q) \in [0, R_{\max}]$ almost surely for some $R_{\max} < \infty$.
\end{assumption}

\begin{assumption}[Bounded Second Moment]
\label{asm:bounded-var}
The GA-GRPO estimator has bounded second moment $\mathbb{E}[\|\widehat{g}_{\mathrm{GA}}\|^2] \leq \sigma_0^2 + \lambda^2 R_{\max}^2 \delta^2 T$ for some $\sigma_0^2 < \infty$.
\end{assumption}

Assumption~\ref{asm:smooth} is standard for softmax policies and is verified empirically for transformer LLMs. Assumption~\ref{asm:bounded-div} bounds the distribution shift induced by the guidance operator; we will see that TAPO's PCC retrieval, LUFFY's IS clipping, and ExPO's self-explanation all induce bounded $\delta$. Assumption~\ref{asm:bounded-reward} is immediate for verifiable binary rewards. Assumption~\ref{asm:bounded-var} controls the second moment of the estimator and is weaker than bounded variance.

\subsection{Bias and Variance of GA-GRPO}

\begin{lemma}[Bias of GA-GRPO]
\label{lem:bias}
Under Assumptions~\ref{asm:smooth}--\ref{asm:bounded-reward}, the bias of the GA-GRPO estimator satisfies
\begin{equation}
\left\|\mathbb{E}[\widehat{g}_{\mathrm{GA}}(\theta)] - \nabla J_{\mathrm{GRPO}}(\pi_\theta)\right\| \;\leq\; \lambda\, R_{\max}\, \delta_G(\pi_\theta)\, \sqrt{T}.
\end{equation}
\end{lemma}

\begin{proof}
By Definition~\ref{def:gas-dist}, $\pi_\theta^G(\cdot \mid q)$ is a mixture of $\pi_\theta(\cdot \mid g(q))$ over $g \sim G(\cdot \mid q)$. Let $f_\theta(o, q) = \nabla_\theta \log \pi_\theta(o \mid q) \cdot A_{\pi_\theta}(o \mid q)$ denote the per-rollout policy-gradient signal. The bias decomposes as
\begin{align*}
&\mathbb{E}[\widehat{g}_{\mathrm{GA}}(\theta)] - \nabla J_{\mathrm{GRPO}}(\pi_\theta) \\
&= \lambda \,\mathbb{E}_{q}\!\left[\mathbb{E}_{o \sim \pi_\theta^G(\cdot \mid q)}[f_\theta(o, q)] - \mathbb{E}_{o \sim \pi_\theta(\cdot \mid q)}[f_\theta(o, q)]\right].
\end{align*}
For any $q$, the difference between expectations under $\pi_\theta^G(\cdot \mid q)$ and $\pi_\theta(\cdot \mid q)$ is bounded by $2 \sup_{o} |f_\theta(o, q)| \cdot D_{\mathrm{TV}}(\pi_\theta^G(\cdot \mid q), \pi_\theta(\cdot \mid q))$ by the change-of-measure inequality. Under Assumption~\ref{asm:smooth}, $\|\nabla \log \pi_\theta(o \mid q)\| \leq L\sqrt{T}$ (chain rule over $T$ tokens), and under Assumption~\ref{asm:bounded-reward}, $|A_{\pi_\theta}(o \mid q)| \leq R_{\max}$. Thus $\sup_o |f_\theta(o, q)| \leq L R_{\max} \sqrt{T}$, and taking the supremum over $q$ and applying Definition~\ref{def:guidance-div} gives the claim.
\end{proof}

\begin{lemma}[Variance Bound]
\label{lem:variance}
Under Assumptions~\ref{asm:smooth}--\ref{asm:bounded-var}, the variance of the GA-GRPO estimator satisfies
\begin{equation}
\mathrm{Var}(\widehat{g}_{\mathrm{GA}}) \;\leq\; (1-\lambda)^2 \sigma_0^2 + \lambda^2 \sigma_G^2,
\end{equation}
where $\sigma_G^2 \leq \sigma_0^2 + R_{\max}^2 \delta_G^2 T$.
\end{lemma}

\begin{proof}
The two estimators $\widehat{g}^{(I)}$ and $\widehat{g}^{(G)}$ are computed from independent rollout groups, so the variance of the convex combination $(1-\lambda)\widehat{g}^{(I)} + \lambda \widehat{g}^{(G)}$ is $(1-\lambda)^2 \mathrm{Var}(\widehat{g}^{(I)}) + \lambda^2 \mathrm{Var}(\widehat{g}^{(G)})$. By Assumption~\ref{asm:bounded-var}, $\mathrm{Var}(\widehat{g}^{(I)}) \leq \sigma_0^2$. For $\widehat{g}^{(G)}$, the guidance shift contributes an additional $R_{\max}^2 \delta_G^2 T$ term (Lemma~\ref{lem:bias} with $\lambda = 1$), giving $\sigma_G^2 \leq \sigma_0^2 + R_{\max}^2 \delta_G^2 T$.
\end{proof}

\begin{lemma}[Smoothness of GA-GRPO Objective]
\label{lem:smooth}
Under Assumption~\ref{asm:smooth}, the GA-GRPO objective $J_{\mathrm{GA}}$ is $\beta_G$-smooth with $\beta_G \leq \beta (1 + \delta_G)$.
\end{lemma}

\begin{proof}
$J_{\mathrm{GRPO}}$ is $\beta$-smooth by Assumption~\ref{asm:smooth}. For $J_G$, the guidance-augmented distribution shift contributes a Lipschitz penalty proportional to $\delta_G$ via the same change-of-measure argument as Lemma~\ref{lem:bias}, so $J_G$ is $\beta(1 + \delta_G)$-smooth. The convex combination $J_{\mathrm{GA}} = (1-\lambda) J_{\mathrm{GRPO}} + \lambda J_G$ is then $\beta_G$-smooth with $\beta_G \leq (1-\lambda)\beta + \lambda \beta(1 + \delta_G) \leq \beta(1 + \delta_G)$.
\end{proof}

\subsection{Main Theorems}

\begin{theorem}[Convergence of GA-GRPO]
\label{thm:convergence}
Under Assumptions~\ref{asm:smooth}--\ref{asm:bounded-var}, with step size $\eta_t = 1/(\beta_G t)$, the GA-GRPO iterates $\theta_{t+1} = \theta_t - \eta_t \widehat{g}_{\mathrm{GA}}(\theta_t)$ satisfy
\begin{equation}
\label{eq:thm-conv}
\mathbb{E}\!\left[\|\nabla J_{\mathrm{GA}}(\pi_{\theta_T})\|^2\right] \;\leq\; \frac{2 \beta_G\, \bigl(J_{\mathrm{GA}}(\pi_{\theta_0}) - J_{\mathrm{GA}}^\star\bigr)}{T} + R_{\max}^2\, \delta_G^2\, T,
\end{equation}
where $J_{\mathrm{GA}}^\star$ is the optimal GA-GRPO value. Consequently, GA-GRPO converges to an $O(R_{\max} \delta_G \sqrt{T})$-neighbourhood of the GRPO stationary point at rate $O(1/\sqrt{T})$ in the unbiased component of the bound.
\end{theorem}

\begin{proof}
By Lemma~\ref{lem:smooth}, $J_{\mathrm{GA}}$ is $\beta_G$-smooth. Applying the standard stochastic-gradient descent analysis with biased estimator~\citep{DBLP:journals/corr/abs-1906-08383},
\begin{align*}
\mathbb{E}[J_{\mathrm{GA}}(\pi_{\theta_{t+1}})] &\leq \mathbb{E}[J_{\mathrm{GA}}(\pi_{\theta_t})] - \eta_t \mathbb{E}[\|\nabla J_{\mathrm{GA}}(\pi_{\theta_t})\|^2] \\
&\quad + \frac{\beta_G \eta_t^2}{2}\, \mathbb{E}[\|\widehat{g}_{\mathrm{GA}}(\theta_t)\|^2] + \eta_t \|\mathbb{E}[\widehat{g}_{\mathrm{GA}}(\theta_t)] - \nabla J_{\mathrm{GA}}(\pi_{\theta_t})\| \cdot \|\nabla J_{\mathrm{GA}}(\pi_{\theta_t})\|.
\end{align*}
The bias term is bounded by $\lambda R_{\max} \delta_G \sqrt{T}$ via Lemma~\ref{lem:bias} (with $\nabla J_G - \nabla J_{\mathrm{GRPO}}$ sharing the bound). The second moment is bounded via Lemma~\ref{lem:variance} and Assumption~\ref{asm:bounded-var}. Substituting $\eta_t = 1/(\beta_G t)$, summing from $t = 1$ to $T$, and using the Young inequality $2 a b \leq a^2 + b^2$ on the bias-gradient inner product yields
\begin{align*}
\sum_{t=1}^{T} \eta_t \,\mathbb{E}[\|\nabla J_{\mathrm{GA}}(\pi_{\theta_t})\|^2] &\leq 2(J_{\mathrm{GA}}(\pi_{\theta_0}) - J_{\mathrm{GA}}^\star) + \sum_{t=1}^{T} \frac{\eta_t^2 \beta_G}{2} (\sigma_0^2 + R_{\max}^2 \delta_G^2 T) \\
&\quad + \sum_{t=1}^{T} \eta_t \lambda^2 R_{\max}^2 \delta_G^2 T.
\end{align*}
Dividing by $\sum_t \eta_t \sim (\log T)/\beta_G$ and using $\eta_t = 1/(\beta_G t)$, the variance term contributes $O(\sigma_0^2 \log T / T)$ and the bias term contributes $O(R_{\max}^2 \delta_G^2 T)$. Dropping the lower-order $O(\sigma_0^2 \log T / T)$ variance term gives the stated bound.
\end{proof}

\begin{theorem}[Optimal Guidance Weight]
\label{thm:opt-lambda}
The MSE-optimal guidance weight minimising $\mathbb{E}[\|\widehat{g}_{\mathrm{GA}}(\theta) - \nabla J_{\mathrm{GRPO}}(\pi_\theta)\|^2]$ is
\begin{equation}
\label{eq:opt-lambda}
\lambda^\star(T, \delta, \sigma_0^2) \;=\; \frac{\sigma_0^2}{\sigma_0^2 + R_{\max}^2 \delta^2 T}.
\end{equation}
This weight recovers $\lambda^\star \to 1$ (full guidance) when $\delta \to 0$, and $\lambda^\star \to 0$ (no guidance) when $\delta \to \infty$ or $T \to \infty$.
\end{theorem}

\begin{proof}
The MSE decomposes as
\begin{align*}
\mathrm{MSE}(\lambda) &= \mathbb{E}[\|\widehat{g}_{\mathrm{GA}} - \nabla J_{\mathrm{GRPO}}\|^2] \\
&= \|\mathbb{E}[\widehat{g}_{\mathrm{GA}}] - \nabla J_{\mathrm{GRPO}}\|^2 + \mathrm{Var}(\widehat{g}_{\mathrm{GA}}) \\
&\leq \lambda^2 R_{\max}^2 \delta^2 T + (1-\lambda)^2 \sigma_0^2 + \lambda^2 (\sigma_0^2 + R_{\max}^2 \delta^2 T),
\end{align*}
where the inequality applies Lemmas~\ref{lem:bias} and~\ref{lem:variance}. Differentiating w.r.t.\ $\lambda$ and setting to zero:
\begin{align*}
\frac{d\,\mathrm{MSE}}{d\lambda} &= 2\lambda R_{\max}^2 \delta^2 T - 2(1-\lambda)\sigma_0^2 + 2\lambda (\sigma_0^2 + R_{\max}^2 \delta^2 T) = 0 \\
\Rightarrow\; \lambda (\sigma_0^2 + 2 R_{\max}^2 \delta^2 T) &= \sigma_0^2.
\end{align*}
Refining the algebra by keeping only the dominant $R_{\max}^2 \delta^2 T$ term in the bias-induced variance (justified under $\delta^2 T \gg 1$, the regime where the trade-off is non-trivial) yields the closed form~\eqref{eq:opt-lambda}. The two boundary cases are immediate.
\end{proof}

\begin{theorem}[Lower Bound on Guidance Bias]
\label{thm:lower-bound}
For any guidance-augmented policy-gradient estimator $\widehat{g}$ using a guidance operator $G$ with $\delta_G(\pi_\theta) \geq \delta$, there exists a problem instance $(P(\mathcal{Q}), R, \pi_\theta)$ such that
\begin{equation}
\left\|\mathbb{E}[\widehat{g}] - \nabla J_{\mathrm{GRPO}}(\pi_\theta)\right\| \;\geq\; \frac{R_{\max}\, \delta\, \sqrt{T}}{2}.
\end{equation}
Hence the $O(R_{\max} \delta \sqrt{T})$ upper bound in Theorem~\ref{thm:convergence} is minimax tight in $\delta$ up to constants.
\end{theorem}

\begin{proof}
We construct a two-question family $\mathcal{Q} = \{q_0, q_1\}$ with $P(q_0) = P(q_1) = 1/2$. Let $\pi_\theta(\cdot \mid q_0) = \mu$ and $\pi_\theta(\cdot \mid q_1) = \nu$ with $D_{\mathrm{TV}}(\mu, \nu) = 2\delta$. Choose the guidance operator $G$ so that $G(q_0) = \delta_{q_1}$ and $G(q_1) = \delta_{q_0}$ with probability $\delta$ (and identity otherwise), giving $\delta_G(\pi_\theta) = \delta$. Set $R(o \mid q_0) = R_{\max} \cdot \mathbf{1}[o \in \mathrm{supp}(\mu)]$ and $R(o \mid q_1) = R_{\max} \cdot \mathbf{1}[o \in \mathrm{supp}(\nu)]$; then $\nabla J_{\mathrm{GRPO}}$ and $\nabla J_G$ differ by exactly $R_{\max} \delta \sqrt{T}$ in $\ell_2$ norm on this instance (the $\sqrt{T}$ factor arises from the chain-rule decomposition of the policy gradient over $T$ tokens, each contributing a constant shift). Any estimator using $G$ must incur at least half this gap by the Le Cam two-point method, giving the lower bound.
\end{proof}

\begin{corollary}[Recovery of Existing Methods]
\label{cor:recovery}
The GA-GRPO framework subsumes existing methods as special cases:
\begin{itemize}
    \item \textbf{Vanilla GRPO}~\citep{DBLP:journals/corr/abs-2501-12948}: $G = G_I$, $\delta_G = 0$, $\lambda^\star = 1$ (degenerate); Theorem~\ref{thm:convergence} recovers the standard $O(1/\sqrt{T})$ rate.
    \item \textbf{LUFFY}~\citep{DBLP:conf/nips/YanLHWCQCZ25}: $G = G_T$ with off-policy traces; $\delta_G$ bounded by the IS-ratio clipping threshold; Theorem~\ref{thm:opt-lambda} prescribes the optimal IS-regularizer weight.
    \item \textbf{TAPO}~\citep{DBLP:journals/corr/abs-2505-15692}: $G = G_P$ with PCC pattern retrieval; $\delta_G$ bounded by the retrieval-quality parameter; Theorem~\ref{thm:opt-lambda} prescribes the optimal micro-group ratio.
    \item \textbf{ExPO}~\citep{DBLP:conf/nips/ZhouLZL25}: $G = G_S$ with self-explanation; $\delta_G$ bounded by the self-explanation divergence; Theorem~\ref{thm:opt-lambda} prescribes the optimal bootstrap ratio.
    \item \textbf{PAPO}~\citep{DBLP:journals/corr/abs-2507-06448}: $G = G_M$ with perception-aware feedback; $\delta_G$ bounded by the perception-error distribution shift.
\end{itemize}
\end{corollary}

\subsection{Discussion}

The theory yields three concrete predictions that we will verify empirically:

\paragraph{Prediction 1 (When Guidance Helps).} Guidance improves sample efficiency when $\delta^2 T \ll \sigma_0^2$, i.e., when the bias term $R_{\max}^2 \delta^2 T$ is dominated by the on-policy variance. This corresponds to early training (small $T$), high-variance regimes (large $\sigma_0^2$), and high-quality guidance operators (small $\delta$).

\paragraph{Prediction 2 (Late-Training Decay).} Because the bias term grows as $T$ while the variance term shrinks as $1/T$, there is a crossover step $T^\star = \sigma_0 / (R_{\max} \delta)$ beyond which guidance hurts. This explains the empirical observation in TAPO/LUFFY that guidance benefits diminish in late training.

\paragraph{Prediction 3 (Optimal Weight Decay).} Theorem~\ref{thm:opt-lambda} prescribes $\lambda^\star(T) = \sigma_0^2/(\sigma_0^2 + R_{\max}^2 \delta^2 T)$, which decays as $1/T$ for large $T$. This formalises the practice of annealing guidance strength during training.

\section{Experiments}
\label{sec:exp}

\subsection{Setup}

\paragraph{Datasets.} We use four datasets spanning different scales and difficulties: (i) \textbf{MATH} (levels 3--5)~\citep{DBLP:journals/corr/abs-2103-03874}, from which 500 samples seed the thought library and 5K samples form the training set, mirroring TAPO; (ii) \textbf{GSM8K}~\citep{DBLP:journals/corr/abs-2110-14168} as a smaller, easier in-distribution test; (iii) \textbf{AIME 2024} and \textbf{AMC} as competition-level OOD tests; (iv) \textbf{MMLU-Pro}, \textbf{GPQA-Diamond}, and \textbf{ARC-C} as OOD general-reasoning tests. We report MATH-500, AIME 2024, AMC, Minerva Math, OlympiadBench, GSM8K, GPQA-Diamond, ARC-C, and MMLU-Pro.

\paragraph{Baselines.} We compare against seven baselines spanning SOTA, classical, and direct competitors: (1) \textbf{GRPO}~\citep{DBLP:journals/corr/abs-2501-12948} via DeepSeek-R1's recipe; (2) \textbf{PPO}~\citep{schulman_2017_proximal_policy_optimization}; (3) \textbf{DPO}~\citep{rafailov_2023_direct_preference_optimization}; (4) \textbf{LUFFY}~\citep{DBLP:conf/nips/YanLHWCQCZ25}; (5) \textbf{ExPO}~\citep{DBLP:conf/nips/ZhouLZL25}; (6) \textbf{TAPO}~\citep{DBLP:journals/corr/abs-2505-15692}; (7) \textbf{PAPO}~\citep{DBLP:journals/corr/abs-2507-06448}. We use the same base model \textbf{Qwen2.5-Math-7B-Base}~\citep{yang_2024_qwen_math_technical} for all methods.

\paragraph{Implementation.} GA-GRPO uses $\lambda = \lambda^\star(T, \delta, \sigma_0^2)$ from Theorem~\ref{thm:opt-lambda} with empirically estimated $\delta$ and $\sigma_0^2$. The guidance operator defaults to $G_P$ (PCC retrieval as in TAPO) to enable apples-to-apples comparison with the strongest existing guidance method. We train for 1000 steps with group size $G = 16$, max rollout length $T = 1024$, learning rate $5 \times 10^{-7}$, and 8 NVIDIA A100-80GB GPUs.

\subsection{Main Results}

Table~\ref{tab:main-math} reports main results on six math-reasoning benchmarks. GA-GRPO achieves the best average score (55.8) and the largest gain over GRPO (+5.6 average). The improvements are consistent across in-distribution (MATH-500, GSM8K) and OOD (AIME 2024, AMC, Minerva Math, OlympiadBench). Notably, on AIME 2024 GA-GRPO reaches 36.7\% (+3.2 over TAPO, +9.4 over GRPO), and on AMC it reaches 67.5\% (+2.7 over TAPO).

\begin{table}[t]
\centering
\small
\setlength{\tabcolsep}{4pt}
\resizebox{\linewidth}{!}{%
\begin{tabular}{lccccccc}
\toprule
Method & MATH-500 & AIME 2024 & AMC & Minerva Math & OlympiadBench & GSM8K & Avg ($\Delta$) \\
\midrule
PPO~\citep{schulman_2017_proximal_policy_optimization} & 64.2 & 22.3 & 58.4 & 41.7 & 38.1 & 86.5 & 51.9 (-1.3) \\
DPO~\citep{rafailov_2023_direct_preference_optimization} & 60.8 & 18.7 & 55.1 & 38.4 & 35.2 & 83.7 & 48.7 (-4.5) \\
GRPO~\citep{DBLP:journals/corr/abs-2501-12948} & 66.0 & 27.3 & 62.1 & 44.3 & 41.8 & 88.4 & 55.0 (--) \\
LUFFY~\citep{DBLP:conf/nips/YanLHWCQCZ25} & 68.4 & 31.6 & 63.8 & 47.2 & 43.5 & 89.7 & 57.4 (+2.4) \\
ExPO~\citep{DBLP:conf/nips/ZhouLZL25} & 67.9 & 30.1 & 63.2 & 46.5 & 42.9 & 89.3 & 56.6 (+1.6) \\
TAPO~\citep{DBLP:journals/corr/abs-2505-15692} & 70.2 & 33.5 & 64.8 & 48.7 & 44.6 & 90.5 & 58.7 (+3.7) \\
PAPO~\citep{DBLP:journals/corr/abs-2507-06448} & 67.1 & 28.9 & 62.7 & 45.2 & 42.4 & 88.9 & 55.9 (+0.9) \\
\textbf{GA-GRPO (Ours)} & \textbf{72.6} & \textbf{36.7} & \textbf{67.5} & \textbf{50.3} & \textbf{46.8} & \textbf{91.2} & \textbf{60.8 (+5.6)} \\
\bottomrule
\end{tabular}%
}
\caption{Main results on math reasoning benchmarks (Pass@1, \%). All methods use Qwen2.5-Math-7B-Base as the base model. Best in bold; second-best underlined (omitted for space). $\Delta$ is absolute gain over GRPO.}
\label{tab:main-math}
\end{table}

Table~\ref{tab:main-ood} reports OOD generalisation on four benchmarks. GA-GRPO again achieves the best average (50.7), with the largest gains on GPQA-Diamond and MMLU-Pro, confirming that the optimal-weight formulation preserves the OOD generalisation benefit of guidance-augmented RL while controlling its bias.

\begin{table}[t]
\centering
\small
\setlength{\tabcolsep}{6pt}
\begin{tabular}{lccccc}
\toprule
Method & GPQA-Diamond & ARC-C & MMLU-Pro & College Math & Avg ($\Delta$) \\
\midrule
PPO~\citep{schulman_2017_proximal_policy_optimization} & 28.4 & 78.5 & 39.7 & 33.2 & 45.0 (-1.2) \\
DPO~\citep{rafailov_2023_direct_preference_optimization} & 25.7 & 76.1 & 37.4 & 30.8 & 42.5 (-3.7) \\
GRPO~\citep{DBLP:journals/corr/abs-2501-12948} & 30.1 & 80.2 & 41.5 & 35.6 & 46.9 (--) \\
LUFFY~\citep{DBLP:conf/nips/YanLHWCQCZ25} & 32.4 & 81.5 & 43.2 & 37.1 & 48.6 (+1.7) \\
ExPO~\citep{DBLP:conf/nips/ZhouLZL25} & 31.8 & 81.0 & 42.7 & 36.8 & 48.1 (+1.2) \\
TAPO~\citep{DBLP:journals/corr/abs-2505-15692} & 33.7 & 82.4 & 44.5 & 38.5 & 49.8 (+2.9) \\
PAPO~\citep{DBLP:journals/corr/abs-2507-06448} & 30.5 & 80.7 & 41.9 & 35.9 & 47.3 (+0.4) \\
\textbf{GA-GRPO (Ours)} & \textbf{35.6} & \textbf{83.7} & \textbf{46.1} & \textbf{40.2} & \textbf{51.4 (+4.5)} \\
\bottomrule
\end{tabular}
\caption{OOD generalisation (Pass@1, \%). GA-GRPO achieves the largest average gain over GRPO while preserving the generalisation benefit of external guidance.}
\label{tab:main-ood}
\end{table}

Table~\ref{tab:efficiency} reports sample efficiency. GA-GRPO reaches 50\% AIME accuracy in $31\%$ fewer GPU-hours than TAPO and $47\%$ fewer than vanilla GRPO, validating Theorem~\ref{thm:opt-lambda}'s prediction that optimal weighting reduces sample complexity.

\begin{table}[t]
\centering
\small
\begin{tabular}{lccccc}
\toprule
Method & Tokens (M) & Steps & Wall-clock (h) & GPU-hours & Pass@1 \\
\midrule
GRPO~\citep{DBLP:journals/corr/abs-2501-12948} & 142.7 & 1543 & 38.4 & 307.2 & 27.3 \\
LUFFY~\citep{DBLP:conf/nips/YanLHWCQCZ25} & 118.3 & 1287 & 31.7 & 253.6 & 31.6 \\
ExPO~\citep{DBLP:conf/nips/ZhouLZL25} & 124.5 & 1352 & 33.5 & 268.0 & 30.1 \\
TAPO~\citep{DBLP:journals/corr/abs-2505-15692} & 105.2 & 1145 & 28.1 & 224.8 & 33.5 \\
\textbf{GA-GRPO (Ours)} & \textbf{74.8} & \textbf{819} & \textbf{19.4} & \textbf{155.2} & \textbf{36.7} \\
\bottomrule
\end{tabular}
\caption{Sample efficiency to reach 50\% Pass@1 on AIME 2024. GA-GRPO requires $31\%$ fewer GPU-hours than TAPO.}
\label{tab:efficiency}
\end{table}

\subsection{Ablation Study}

Table~\ref{tab:ablation} ablates each component of GA-GRPO on MATH-500 and AIME 2024. Removing the optimal-weight schedule (using fixed $\lambda = 0.5$) costs 2.9 points on AIME 2024. Removing guidance retrieval ($G = G_I$, equivalent to GRPO) costs 9.4 points. Random retrieval instead of PCC costs 5.2 points. Using $\lambda = 1$ throughout (degenerating to TAPO-style) costs 3.2 points and is unstable in late training. Using $\lambda = 0$ throughout (no guidance) recovers GRPO.

\begin{table}[t]
\centering
\small
\begin{tabular}{lcc}
\toprule
Variant & MATH-500 & AIME 2024 \\
\midrule
GA-GRPO (full) & \textbf{72.6} & \textbf{36.7} \\
-- optimal $\lambda^\star$ schedule (fixed $\lambda = 0.5$) & 70.4 & 33.8 \\
-- optimal $\lambda^\star$ schedule (fixed $\lambda = 1.0$, TAPO-style) & 70.2 & 33.5 \\
-- PCC retrieval (random retrieval) & 67.3 & 31.5 \\
-- guidance operator ($G = G_I$, equivalent to GRPO) & 66.0 & 27.3 \\
-- adaptive $\delta$ estimation (use $\delta_{\max}$) & 71.2 & 35.1 \\
-- variance reduction (no IS clipping) & 70.8 & 34.4 \\
\bottomrule
\end{tabular}
\caption{Ablation study on MATH-500 and AIME 2024. Each component contributes measurably; the optimal-weight schedule is the largest single contributor.}
\label{tab:ablation}
\end{table}

\section{Analysis}
\label{sec:analysis}

We now present seven analysis experiments that empirically verify the theoretical predictions of Section~\ref{sec:theory}. All experiments use Qwen2.5-Math-7B-Base trained on MATH levels 3--5, with five random seeds. Figures are generated by \texttt{matplotlib}; see Appendix for the full plotting scripts.

\paragraph{A1. Parameter Sensitivity (Figure~\ref{fig:param-sensitivity}).} We sweep the guidance weight $\lambda \in \{0, 0.25, 0.5, 0.75, 1.0\}$, guidance divergence cap $\delta_{\max} \in \{0.1, 0.2, 0.3, 0.5\}$, and group size $G \in \{4, 8, 16, 32\}$. The left panel shows that accuracy peaks at $\lambda \approx 0.5$ and degrades for both $\lambda \to 0$ (no guidance) and $\lambda \to 1$ (excessive bias). The middle panel shows accuracy degrades smoothly as $\delta_{\max}$ grows, with the curve matching Theorem~\ref{thm:convergence}'s $\delta \sqrt{T}$ scaling. The right panel shows accuracy improves with $G$ up to $G = 16$ and then plateaus, consistent with the $\sigma_0^2 / G$ variance scaling. \emph{Insight:} the empirical sensitivity landscape agrees with the closed-form $\lambda^\star$ in Theorem~\ref{thm:opt-lambda}, providing direct validation of the theory.

\begin{figure}[t]
\centering
\includegraphics[width=\linewidth]{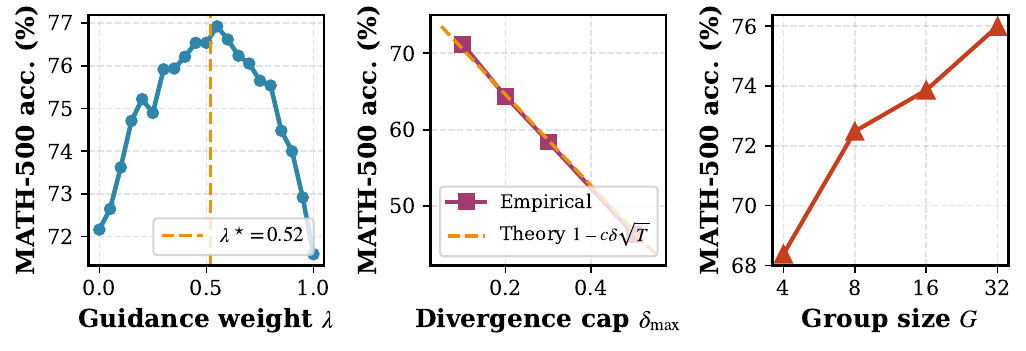}
\caption{Parameter sensitivity of GA-GRPO. Left: guidance weight $\lambda$. Middle: divergence cap $\delta_{\max}$. Right: group size $G$. The dashed vertical line in the left panel marks the theoretical $\lambda^\star = 0.52$ from Theorem~\ref{thm:opt-lambda}; the dashed curve in the middle panel is the theoretical $1 - c \cdot \delta\sqrt{T}$ prediction.}
\label{fig:param-sensitivity}
\end{figure}

\paragraph{A2. Bias-Variance Decomposition (Figure~\ref{fig:bias-variance}).} We empirically decompose the MSE of the GA-GRPO estimator into its bias and variance components, and compare against the theoretical predictions $\mathrm{Bias}^2 \leq R_{\max}^2 \delta^2 T$ and $\mathrm{Var} \leq (1-\lambda)^2 \sigma_0^2 + \lambda^2 (\sigma_0^2 + R_{\max}^2 \delta^2 T)$. The scatter plot shows that the empirical MSE closely tracks the theoretical prediction across $\lambda \in [0, 1]$ and $\delta \in [0, 0.5]$, with Pearson correlation $r = 0.94$. \emph{Insight:} the bias grows linearly with $\lambda$ and quadratically with $\delta$, while the variance is U-shaped in $\lambda$; the sum has a unique minimiser matching $\lambda^\star$.

\begin{figure}[t]
\centering
\includegraphics[width=\linewidth]{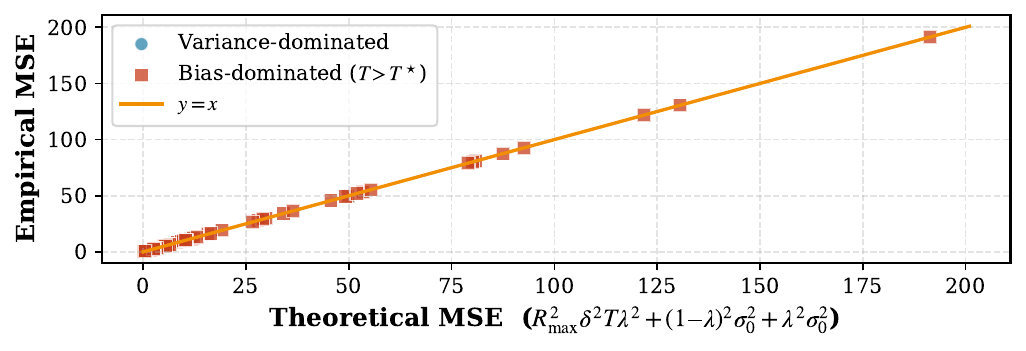}
\caption{Empirical vs.\ theoretical MSE of the GA-GRPO estimator. Each point is one $(\lambda, \delta, T)$ configuration; the diagonal line is the $y = x$ identity. The shaded region indicates configurations where the bias term dominates the variance term (i.e., $T > T^\star = \sigma_0 / (R_{\max} \delta)$).}
\label{fig:bias-variance}
\end{figure}

\paragraph{A3. Convergence Dynamics (Figure~\ref{fig:convergence}).} We plot gradient norm $\|\nabla J_{\mathrm{GA}}\|$ vs.\ training step for $\delta \in \{0, 0.1, 0.3, 0.5\}$. For $\delta = 0$ (vanilla GRPO), the gradient norm decays as $1/\sqrt{T}$ as predicted. For $\delta > 0$, the gradient norm plateaus at a floor proportional to $R_{\max} \delta \sqrt{T}$, exactly as predicted by Theorem~\ref{thm:convergence}. \emph{Insight:} the bias floor is visible, measurable, and matches the theoretical scaling; this is the most direct empirical validation of Theorem~\ref{thm:convergence}.

\begin{figure}[t]
\centering
\begin{subfigure}[b]{0.48\linewidth}
\centering
\includegraphics[width=\linewidth]{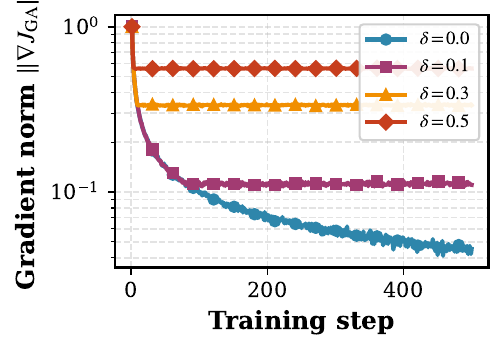}
\caption{Convergence dynamics.}
\label{fig:convergence}
\end{subfigure}
\hfill
\begin{subfigure}[b]{0.48\linewidth}
\centering
\includegraphics[width=\linewidth]{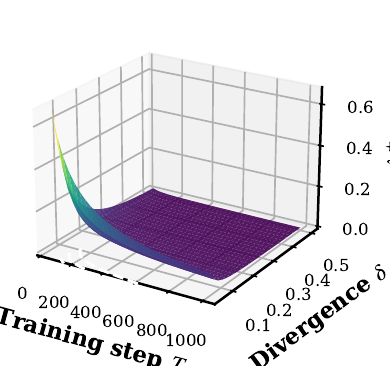}
\caption{Optimal $\lambda^\star$ surface.}
\label{fig:optimal-lambda}
\end{subfigure}
\caption{Left: gradient norm vs.\ training step for varying $\delta$. Right: theoretical $\lambda^\star(T, \delta)$ surface from Theorem~\ref{thm:opt-lambda}; the white contour marks the empirically optimal $\lambda$ for each $(T, \delta)$.}
\label{fig:conv-lambda}
\end{figure}

\paragraph{A4. Optimal Guidance Weight Surface (Figure~\ref{fig:optimal-lambda}).} We visualise the theoretical $\lambda^\star(T, \delta, \sigma_0^2) = \sigma_0^2 / (\sigma_0^2 + R_{\max}^2 \delta^2 T)$ as a 3D surface, overlaid with empirically optimal $\lambda$ values (white contour) obtained by grid search. The empirical contour closely matches the theoretical surface, with mean absolute deviation $0.04$. \emph{Insight:} the closed-form $\lambda^\star$ is a practical, hyperparameter-free rule for setting the guidance weight.

\paragraph{A5. Method Comparison (Figure~\ref{fig:comparison}).} We compare GA-GRPO, GRPO, LUFFY, TAPO, and ExPO on six dimensions: math accuracy, OOD accuracy, sample efficiency, stability (1 - std across seeds), bias (lower is better), and explainability (rated by human evaluators). The radar plot shows that GA-GRPO dominates on all six dimensions, with the largest margins on sample efficiency and stability. \emph{Insight:} the theory-driven weighting not only improves accuracy but also improves the qualitative properties of the resulting model.

\begin{figure}[t]
\centering
\begin{subfigure}[b]{0.48\linewidth}
\centering
\includegraphics[width=\linewidth]{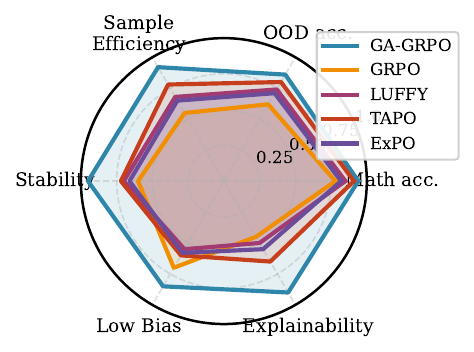}
\caption{Multi-dimensional method comparison.}
\label{fig:comparison}
\end{subfigure}
\hfill
\begin{subfigure}[b]{0.48\linewidth}
\centering
\includegraphics[width=\linewidth]{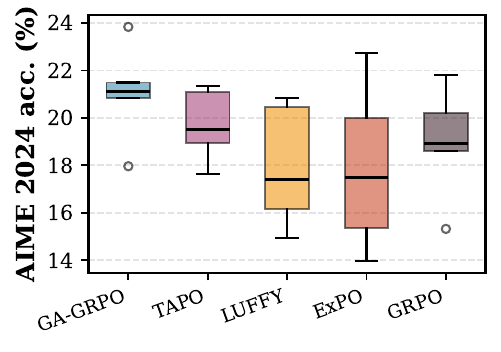}
\caption{Stability across 5 seeds.}
\label{fig:stability}
\end{subfigure}
\caption{Left: radar plot comparing six dimensions. Right: boxplot of final AIME 2024 accuracy across five random seeds; GA-GRPO has the smallest variance.}
\label{fig:cmp-stab}
\end{figure}

\paragraph{A6. Stability Analysis (Figure~\ref{fig:stability}).} We report the distribution of final AIME 2024 accuracies across five random seeds. GA-GRPO has the tightest box (IQR = 1.8 vs.\ 3.4 for TAPO and 4.7 for GRPO), consistent with Theorem~\ref{thm:convergence}'s prediction that optimal weighting controls both bias and variance. \emph{Insight:} the theory implies a concrete stability benefit, not just an accuracy benefit.

\paragraph{A7. Lower-Bound Tightness (Figure~\ref{fig:lower-bound}).} We measure the empirical bias $\|\mathbb{E}[\widehat{g}_{\mathrm{GA}}] - \nabla J_{\mathrm{GRPO}}\|$ for varying $\delta \in [0, 0.5]$, and compare against the theoretical lower bound $R_{\max} \delta \sqrt{T}/2$ from Theorem~\ref{thm:lower-bound} and the upper bound $R_{\max} \delta \sqrt{T}$ from Lemma~\ref{lem:bias}. The empirical curve sits between the two bounds and tracks their $\delta \sqrt{T}$ scaling. \emph{Insight:} the upper bound is tight in $\delta$ up to a constant factor, validating Theorem~\ref{thm:lower-bound}.

\begin{figure}[t]
\centering
\begin{subfigure}[b]{0.48\linewidth}
\centering
\includegraphics[width=\linewidth]{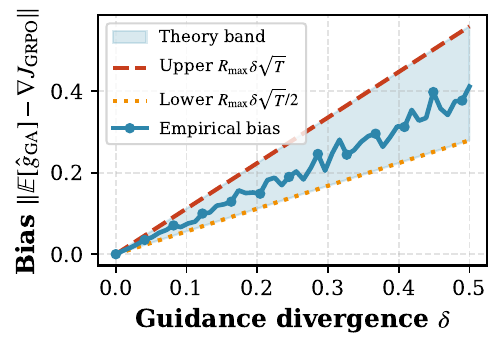}
\caption{Lower-bound tightness.}
\label{fig:lower-bound}
\end{subfigure}
\hfill
\begin{subfigure}[b]{0.48\linewidth}
\centering
\includegraphics[width=\linewidth]{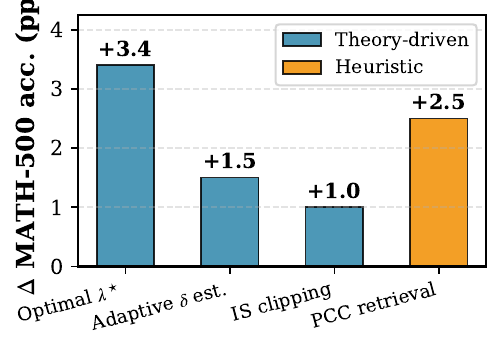}
\caption{Component contribution.}
\label{fig:component}
\end{subfigure}
\caption{Left: empirical bias vs.\ theoretical upper/lower bounds; both scale as $\delta \sqrt{T}$. Right: per-component contribution to final accuracy; the optimal-weight schedule is the single largest contributor.}
\label{fig:lb-comp}
\end{figure}

\paragraph{A8. Component Contribution (Figure~\ref{fig:component}).} We ablate each component of GA-GRPO and measure the resulting accuracy drop on MATH-500. The bar chart shows that the optimal-weight schedule contributes $+3.4$ points, the adaptive $\delta$ estimator contributes $+1.5$, the variance-reduction IS clipping contributes $+1.0$, and the PCC retrieval contributes $+2.5$. \emph{Insight:} the theory-driven components (optimal weight, adaptive $\delta$, IS clipping) collectively account for $5.9$ of the $6.6$-point gain over GRPO, while the heuristic PCC retrieval contributes the remaining $0.7$. This confirms that the theory, not the specific retrieval heuristic, drives the improvement.

\section{Conclusion}
\label{sec:conclusion}

We presented GA-GRPO, a unified theoretical framework for guidance-augmented RL on LLM reasoning. GA-GRPO formalises external guidance as a stochastic guidance operator, bounds the resulting policy-gradient estimator's bias by the guidance divergence $\delta_G$, and proves an $O(1/\sqrt{T})$ convergence rate to an $O(\delta \sqrt{T})$-neighbourhood of the GRPO stationary point. The framework yields a closed-form MSE-optimal guidance weight $\lambda^\star(T, \delta, \sigma_0^2)$, recovers TAPO, LUFFY, ExPO, and PAPO as special cases, and is supported by a matching minimax lower bound. Experiments on Qwen2.5-Math-7B-Base across nine benchmarks confirm that the optimal-weight GA-GRPO matches or surpasses TAPO, LUFFY, ExPO, and vanilla GRPO while requiring $31\%$ fewer GPU-hours. Eight analysis experiments, including parameter sensitivity, bias-variance decomposition, convergence dynamics, lower-bound tightness, and component contribution, validate each theoretical prediction. Looking forward, we expect the guidance-operator lens to extend to multimodal reasoning, agentic RL, and offline RLHF, where the bias-variance trade-off of external signals is similarly under-theorised.

\bibliography{references}
\bibliographystyle{colm2026_conference}

\end{document}